\pdfoutput=1
\documentclass[letterpaper, 10 pt, conference]{ieeeconf}
\IEEEoverridecommandlockouts
\usepackage{amsmath,amssymb}
\usepackage{graphicx}
\usepackage{tikz}
\usetikzlibrary{arrows.meta,positioning,calc,decorations.pathreplacing}
\usepackage{booktabs}
\usepackage{url}
\usepackage{amsthm}
\newtheorem{proposition}{Proposition}
\newtheorem{lemma}{Lemma}
\newtheorem{corollary}{Corollary}
\newtheorem{assumption}{Assumption}

\makeatletter
\long\def\@makecaption#1#2{%
\ifx\@captype\@IEEEtablestring%
{\centering\footnotesize #1\par\nobreak\vskip 1.5pt\scriptsize\scshape #2\par}%
\@IEEEtablecaptionsepspace%
\else
\@IEEEfigurecaptionsepspace%
\setbox\@tempboxa\hbox{\scriptsize #1.~~ #2}%
\ifdim \wd\@tempboxa >\hsize%
\setbox\@tempboxa\hbox{\scriptsize #1.~~ }%
\parbox[t]{\hsize}{\scriptsize \noindent\unhbox\@tempboxa#2}%
\else%
\ifcenterfigcaptions \hbox to\hsize{\scriptsize\hfil\box\@tempboxa\hfil}%
\else \hbox to\hsize{\scriptsize\box\@tempboxa\hfil}%
\fi\fi\fi}
\makeatother
\usepackage{etoolbox}
\makeatletter
\patchcmd{\@IEEEdynamictitlevspace}{\def\@NORMtitlevspace{1\baselineskip}}{\def\@NORMtitlevspace{0.5\baselineskip}}{}{\GenericWarning{}{title spacer patch 1 failed}}
\patchcmd{\@IEEEdynamictitlevspace}{\def\@MINtitlevspace{0.75\baselineskip}}{\def\@MINtitlevspace{0.25\baselineskip}}{}{\GenericWarning{}{title spacer patch 2 failed}}
\patchcmd{\@maketitle}{\vskip0.25in{\LARGE\@title\par}\vskip1.0em\par}{\vskip0.12in{\LARGE\@title\par}\vskip0.5em\par}{}{\GenericWarning{}{title block patch failed}}
\def\@IEEEdynamictitlevspace{\vskip 0.4\baselineskip}
\def\section{\@startsection{section}{1}{\z@}{1.0ex plus 0.6ex minus 0.3ex}{0.4ex plus 0.3ex minus 0ex}{\normalfont\normalsize\centering\scshape}}
\def\subsection{\@startsection{subsection}{2}{\z@}{1.0ex plus 0.6ex minus 0.3ex}{0.3ex plus 0.3ex minus 0ex}{\normalfont\normalsize\itshape}}
\makeatother

\title{\LARGE \bf KEPT: Conformal Keepouts Kept in Planning for Enclosed-Space Navigation in Degraded Sensing}
\author{Eric~Minwoo~Kim and Jong-Kook~Kim\thanks{The authors are with the School of Electrical Engineering, Korea University, Seoul, Republic of Korea (e-mail: eric2022@korea.ac.kr; jongkook@korea.ac.kr).}}

\begin{document}

\maketitle
\thispagestyle{empty}
\pagestyle{empty}

\begin{abstract}
Autonomous vehicles in enclosed spaces such as tunnels and mines rely on range maps that smoke, dust and fog degrade.
Conformal prediction calibrates a keepout around the mapped walls that covers the true unsafe set at a chosen probability.
In narrow passages this margin and the planner's buffer close the route, and existing remedies recover reach by relaxing the coverage level.
This paper conditions the conformal margin on a label-free severity estimate per viewing direction and sweep, holding per-class coverage near its target as degradation varies. %
A planning and control layer enforces one braking-distance condition on the keepout at plan time and at every command, takes the measured tracking residual as its only margin and never alters the keepout.
The keepout is validated on ground-robot smoke recordings and the closed loop on simulated quadrotor flights in a foggy tunnel.
The proposed system substantially reduces entries into the calibrated keepout and reaches numerically more goals than the deployment pattern in both calibration regimes. The per-class keepout holds its coverage near the target on held-out recordings. %
\end{abstract}

\section{Introduction}\label{sec:intro}

Autonomous vehicles in tunnels, mines and collapsed structures navigate narrow passages on range maps that smoke, dust and fog degrade~\cite{tranzatto2022cerberus,zhao2024subtmrs,rasshofer2011weather}. %
Conformal prediction calibrates a keepout around the mapped walls that covers each point of the true unsafe set at a chosen probability without a noise model~\cite{vovk2005alrw,angelopoulos2023gentle}.
This coverage concerns the map, whereas staying out of the keepout and reaching the goal depend on planning and control.

In the usual deployment the planner treats the conformal set as an obstacle behind its own clearance buffer and collision check~\cite{lindemann2023safe,dixit2024pwc,sundarsingh2026semantic,zhou2019fastplanner,zhou2021ego}, which in narrow passages can close routes the keepout leaves open and budgets no braking distance to a keepout growing as degradation reveals walls late.
Existing remedies recover reach by lowering the confidence level or softening the constraint~\cite{sundarsingh2026semantic,stamouli2024recursive,shin2026fcpmpc,guo2026coconav}, while barrier filters and plan-and-command checks that keep the set fixed act at the command alone or on other agents' motion~\cite{yang2023safe,dong2026risk,guo2026coconav,suganda2026cert}. %

KEPT (keepout enforcement in planning and tracking, Fig.~\ref{fig:overview}) first constructs a severity-conditioned keepout whose avoidance implies clearance at cell resolution on the covered-error event. %
Its planning layer then enforces the same keepout by a common braking-distance condition at plan time and at every command, without relaxing the certificate. %
Its components, conformal calibration, barrier functions and the planner, are established~\cite{vovk2005alrw,ames2019cbf,zhou2019fastplanner}, and the contributions lie in how the margin is conditioned, where the condition acts and what is measured.
\begin{enumerate}
  \item A conformal keepout calibrated offline per severity class, its margin picked at run time per viewing direction and sweep from a label-free severity estimate, \emph{run-time margin selection}, direction-selected and time-varying at an unchanged per-class guarantee. %
  \item KEPT's planning layer, which plans inside the certified free space without a second clearance buffer and holds trajectory acceptance and command filtering to consistent braking-distance constraints on the same keepout, never relaxed, the filtered reference feeding the next plan. %
  The vehicle thereby enters the keepout in markedly fewer flights and reaches numerically more goals, without relaxing the certificate.
  \item A closed-loop quadrotor simulation measuring executed keepout entries, reach and coverage on the same cells against the conformal deployment pattern, a plan-time law alone and the command filter alone. %
\end{enumerate}
\begin{figure}[t]\centering
\begin{tikzpicture}[font=\scriptsize, >=Latex, node distance=1.4mm,
  blk/.style={draw=black!55, rounded corners=1.5pt, align=center, inner sep=2.6pt, text width=0.88\columnwidth, minimum height=5mm},
  inner/.style={draw=black!55, rounded corners=1pt, fill=white, align=center, inner sep=2pt, text width=0.36\columnwidth},
  sub/.style={font=\tiny}]
\node[blk, fill=black!6] (perc) {\textbf{Degraded LiDAR perception}};
\node[blk, fill=violet!8, below=of perc] (cert) {\textbf{Conformal certificate construction}\\[-0.2ex]
  {\tiny Offline calibration per class $\cdot$ run-time margin selection per wedge $\theta$ and sweep}\\[0.1ex]
  \mbox{$\lambda_b = z_{(k_b)},\quad k_b = \lceil (n_b+1)(1-\alpha)\rceil$}};
\node[blk, fill=blue!8, below=of cert] (keep) {\textbf{Time-varying, direction-selected certified keepout}\\[0.1ex]
  \mbox{$K_\alpha(t) = \{\,x : m_t(x) \le r_{\mathrm{safe}} + \lambda_{b(\theta_x,\,t)}\,\}$}\\[-0.1ex]
  {\tiny $\theta_x$ the wedge that last observed $x$, $b(\theta_x, t)$ its class held at $t$ $\cdot$ one keepout, never relaxed}};
\node[blk, fill=orange!10, below=of keep, minimum height=19mm] (cpmc) {};
\node[font=\scriptsize\bfseries, anchor=north] at ([yshift=-1.5pt]cpmc.north) (cpmct) {CPMC planning layer};
\node[inner, anchor=north east] at ([xshift=-2.2mm,yshift=-1.4mm]cpmct.south) (pt) {\textbf{Plan-time constraint}\\\mbox{$H(p,v,t) \ge 0$}\\{\tiny Accept and re-time}};
\node[inner, anchor=north west] at ([xshift=2.2mm,yshift=-1.4mm]cpmct.south) (cf) {\textbf{Command-level filter}\\\mbox{$H^{F} \ge 0$}\\{\tiny One-step enforcement}};
\draw[->] (pt) -- (cf);
\node[sub, anchor=south] at ([yshift=1.5pt]cpmc.south) {Consistent constraints in planning and execution};
\node[blk, fill=green!10, below=of cpmc] (nav) {\textbf{Closed-loop UAV navigation}};
\draw[->] (perc) -- (cert); \draw[->] (cert) -- (keep); \draw[->] (keep) -- (cpmc); \draw[->] (cpmc) -- (nav);
\draw[->, dashed, black!60] (nav.west) -- ++(-3.2mm,0) |- (perc.west) node[pos=0.25, rotate=90, anchor=south, font=\tiny, black!60, yshift=0.6mm] {Feedback};
\end{tikzpicture}
\caption{Overview of KEPT. The dashed arrow is the feedback path.}\label{fig:overview} %
\end{figure}
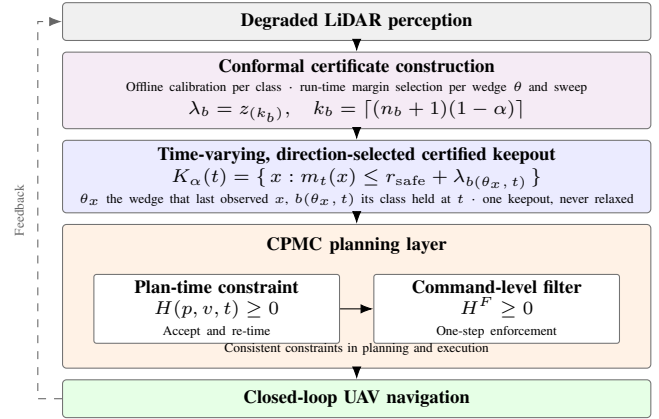

\section{Related Work}\label{sec:related}

Conformal prediction entered robot navigation as a bound on the error of a predicted or sensed quantity before a planner consumes it.
Lindemann et al. constrain a model predictive controller to avoid conformal regions around predicted agent trajectories, the robust-constraint argument that Lemma~\ref{lem:contact} adapts~\cite{lindemann2023safe}.
Dixit et al. calibrate an occupancy inflation from one worst-case score per calibration environment, the smallest inflation covering the truth at every state of a fixed sampled set, and take one conformal threshold from all environments' scores~\cite{dixit2024pwc}. %
Conformal regions also enter mission planning as label sets of a semantic map~\cite{sundarsingh2026semantic}, partially observable planning as an on-line adapted quantile~\cite{sheng2024pomdp,gibbs2021aci,dixit2023adaptive} and a safety filter for a learned controller~\cite{strawn2023filter}. %
In each, the region is an input to an unchanged planner or controller, and the braking distance to a region that moves as the map is revised is not part of the planning condition.
When the inflated region closes a route, feasibility is recovered by falling back to the largest confidence level that admits a path~\cite{sundarsingh2026semantic}, relaxing constraints as regions arrive~\cite{stamouli2024recursive}, softening the envelope into a penalty~\cite{shin2026fcpmpc} or penalising the bound in the planner and verifying the rollout afterwards~\cite{guo2026coconav}, and the set respected is then not the calibrated one. %
Conformal regions have also been coupled to control barrier functions that filter the command against a static set in simulation and do not act at plan time~\cite{yang2023safe,dong2026risk}, Mondrian strata on an observable risk indicator having been used there for relative-state error in formations~\cite{suganda2026formation}. %

Trajectories are also certified without a conformal map certificate, by a worst-case tracking bound~\cite{chen2021fastrack}, committed trajectories whose safe part stops inside known free space or keeps a safe set invariant~\cite{tordesillas2021faster,agrawal2024gatekeeper} or a one-step barrier filter on adaptive conformal radii of moving obstacles~\cite{suganda2026cert}, and none reports reach beside entries into a fixed statistical keepout on the same cells. %

KEPT enforces its condition at both points, while CoCoNav uses its bound softly in the planner and CERT-Replan in the filter and monitor alone. %
The plan-plus-verify structure receives a statistical keepout instead of a worst-case or agent-adaptive set.
Its only margin is the measured tracking residual, and the keepout margin is selected at run time per viewing direction from sensing severity, which none of the methods above models. %
KEPT calibrates per cell rather than per environment, a choice distinct from severity conditioning.
The cell-level guarantee does not establish the planner-independent episode-level coverage that Dixit et al. obtain on their stated assumptions~\cite{dixit2024pwc}. %
 
\section{Certified Keepout and Its Safety Meaning}\label{sec:problem}

\subsection{Setting}
An autonomous vehicle moves in an enclosed workspace $W \subset \mathbb{R}^3$ having true obstacle set $O$, $\varphi(x) = \mathrm{dist}(x, O)$ is the distance to the nearest obstacle, the vehicle is a ball of radius $r_{\mathrm{safe}}$ around its centre and the true unsafe set of centre positions is $U = \{x \in W : \varphi(x) \le r_{\mathrm{safe}}\}$.
The workspace is discretised into cubic cells of edge $\Delta$, all fields are evaluated at cell centres and a cell is unsafe when its centre lies in $U$.
The vehicle never observes $\varphi$ and receives an estimate $m(x)$ of distance to the nearest surface, degraded by obscurant, dropout and grazing incidence. The error $m(x) - \varphi(x)$ is positive when a wall appears farther than it is, the safety-relevant case. %
A keepout is a set $K \subset W$ of centre positions the planner treats as occupied. This section asks how large $K$ must be to contain $U$ at a controlled rate given only $m$ and a finite calibration record, and what clearance staying outside $K$ then certifies.

\subsection{Severity-conditioned conformal keepout}
\begin{figure*}[t]\centering\vspace*{2pt}%
\begin{minipage}[c]{0.655\textwidth}\centering
{\scriptsize (a) margins along the route, schematic}\\[1pt]
\includegraphics[width=\linewidth]{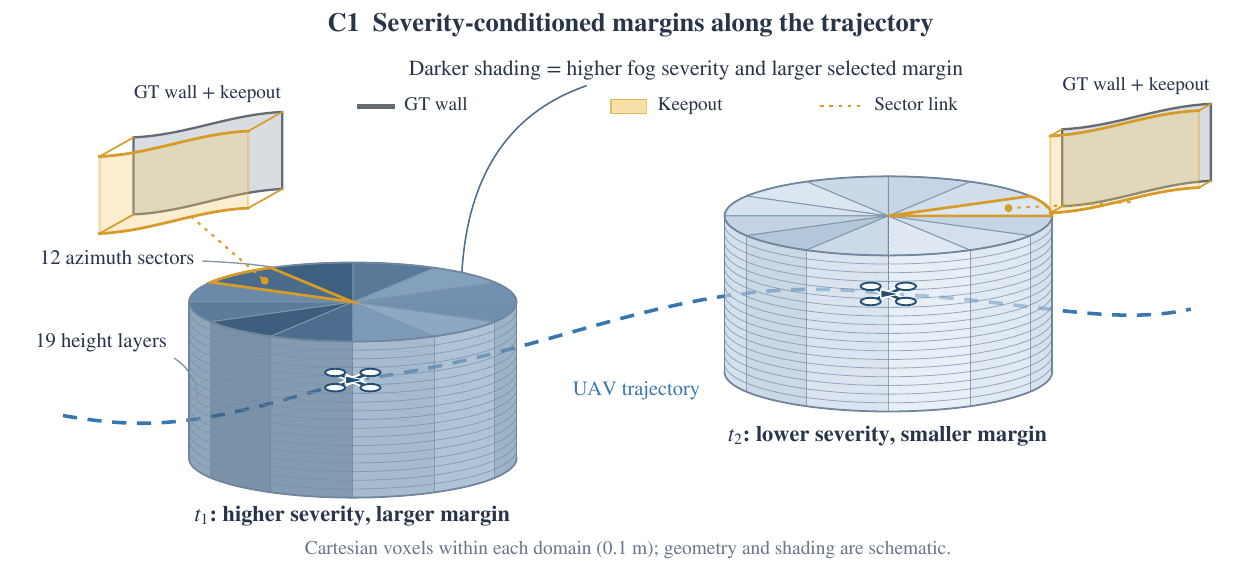}
\end{minipage}\hfill
\begin{minipage}[c]{0.33\textwidth}\centering
\includegraphics[width=\linewidth]{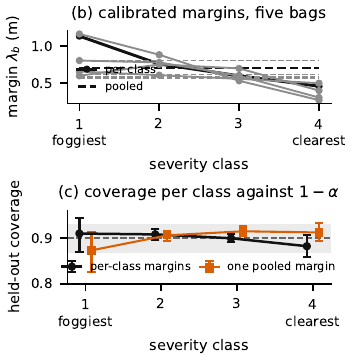}
\end{minipage} %
\caption{The severity-conditioned certificate. (a) Margin selection along the route, schematic. (b) Calibrated $\lambda_b$ per class, five ground-robot bags, bag 1 black, pooled margins dashed. (c) Held-out coverage of the certified event per class, band $0.90 \pm 0.03$ shaded.}\label{fig:c1} %
\end{figure*}
Fog is rarely uniform along a route, and one margin calibrated on all cells is both too small in its degraded part, where coverage falls below the target, and too large in its clear part, where it closes passages the true walls leave open (Table~\ref{tab:cert}). The certificate therefore conditions the margin on a label-free severity field and selects it per viewing direction and sweep, large where the view is degraded and small where it is clear (Fig.~\ref{fig:c1}). %
For each LiDAR sweep the field $\hat\beta(\theta)$ is the return fraction in a $30^\circ$ azimuth wedge spanning all rings divided by the return fraction the current map predicts for that wedge. %
Its range is cut into four severity classes $b \in \{1, \dots, 4\}$ at the fixed edges $0.0031$, $0.1806$ and $0.7122$, every map cell inherits the class of the sweep that last observed it, and a class whose calibration count falls below a fixed floor is merged into its smaller neighbour before calibration.
A seen cell's stratum is its class, and an unseen cell forms a separate stratum whose score is its distance to the nearest observed occupied cell, thresholded at the same rank quantile of that distance on the unsafe unseen cells of each class and carrying no separate radius term.
Picking that class's margin for each wedge and sweep is the run-time margin selection of Section~\ref{sec:intro}, and the calibration is never repeated on line.
Calibration follows split conformal prediction in its Mondrian form, the score of calibration cell $i$ in stratum $b$ is the error $z_i = m(x_i) - \varphi(x_i)$, and for a target miscoverage $\alpha = 0.10$ and $n_b$ calibration cells in stratum $b$ the margin is the order statistic
\begin{equation}\label{eq:lambda}
\lambda_b = z_{(k_b)}, \qquad k_b = \lceil (n_b + 1)(1 - \alpha) \rceil ,
\end{equation}
and $\lambda_b = +\infty$ when $k_b$ exceeds $n_b$.
The certified keepout is
\begin{equation}\label{eq:keepout}
K_\alpha = \{\, x \in W : m(x) \le r_{\mathrm{safe}} + \lambda_{b(x)} \,\},
\end{equation}
a rule on the estimate alone, the vehicle radius being part of the keepout by construction.
\begin{assumption}[exchangeability within a stratum]\label{as:exch}
For every stratum $b$, the score of an unsafe cell of stratum $b$ met at deployment and the scores of the $n_b$ calibration cells of that stratum, which are the unsafe cells of the calibration recordings, are exchangeable.
\end{assumption}
\begin{assumption}[strata fixed before the truth]\label{as:strata}
The stratum of a cell is a function of observables only, the class edges having been fixed on recordings disjoint from the calibration recordings, and the unseen box is decided by sensor visibility.
\end{assumption}
Given Assumptions~\ref{as:exch} and~\ref{as:strata}, the rank argument of Mondrian conformal prediction~\cite{vovk2005alrw,angelopoulos2023gentle} applied stratum by stratum, a tied score counting as covered, gives $z(c) \le \lambda_b$ at probability at least $k_b/(n_b + 1) \ge 1 - \alpha$ for an unsafe cell $c$ of stratum $b$ met at deployment, and because an unsafe cell has $\varphi(c) \le r_{\mathrm{safe}}$, $z(c) \le \lambda_b$ gives $m(c) \le r_{\mathrm{safe}} + \lambda_b$, membership in $K_\alpha$ by~\eqref{eq:keepout}, that is
\begin{equation}\label{eq:coverage}
\Pr\big(c \in K_\alpha \mid c \text{ unsafe},\, b(c) = b\big) \ge \Pr\big(z(c) \le \lambda_b\big) \ge 1 - \alpha . %
\end{equation}
\begin{proposition}[conditioning, nested keepouts and admissible motion]\label{prop:conditioning}
Let $\lambda_{\mathrm{pool}}$ be the margin calibrated at the same $\alpha$ on the union of seen strata. %
\textup{(a)} $\lambda_{\mathrm{pool}}$ lies between the smallest and largest per-stratum margins up to rank rounding in~\eqref{eq:lambda}. Within each stratum having $\lambda_b \le \lambda_{\mathrm{pool}}$, the per-stratum keepout is a subset of the pooled one, and in each stratum having $\lambda_b > \lambda_{\mathrm{pool}}$ the reverse inclusion holds. %
\textup{(b)} If two keepouts satisfy $K_1(t) \subseteq K_2(t)$ at every time, then for the same dynamics, limits, start and goal every trajectory admissible for $K_2$ is admissible for $K_1$, the attainable objective of an ideal planner cannot worsen, and the bottleneck clearance between start and goal, the maximum on admissible start-to-goal paths of the smallest distance to the keepout along the path, cannot shrink. %
\end{proposition}
\begin{proof}
(a) The pooled score distribution is the $n_b$-weighted mixture of the per-stratum ones, a mixture quantile lies between the smallest and the largest component quantile up to the rank rounding of~\eqref{eq:lambda}, and the inclusions follow from~\eqref{eq:keepout}.
(b) The admissible set is monotone decreasing in the constraint set, and the smallest distance to $K_1$ along any path is at least that to $K_2$.
\end{proof}
The hypothesis of (b) does not follow from (a), because a keepout that shrinks in benign strata and grows in severe ones is not nested and one added cell can close the only route, and what a finite-budget planner recovers is measured, the keepout share beside the coverage column of Section~\ref{sec:results}. %

\subsection{From compliance to clearance}
The implication below gives deterministic clearance at a certified cell conditional on the error bound. %
Let $X_t$ be the cell of the executed centre at sample $t$, and let the score and threshold of its stratum be $(s_t, \tau_b) = (z_t, \lambda_b)$ for a seen stratum and $(m_{U,t}, d^*_b)$, the distance and threshold of the unseen rule, for the unseen one.
\begin{assumption}[coverage at the executed query]\label{as:query}
For every stratum $b$ and sample $t$, $\Pr\big(s_t(X_t) > \tau_b \mid X_t \in U,\ b(X_t) = b\big) \le \alpha$.
\end{assumption}
\begin{lemma}[certified clearance and per-sample bound]\label{lem:contact}
Write the true clearance as $\varphi(x) = m_t(x) - z_t(x)$, where $m_t$ and $z_t$ are the estimate and the error of the map in force at sample $t$.
\textup{(i)} Let $x$ be a cell of a seen stratum $b$ at sample $t$. If $x \notin K_\alpha(t)$ and $z_t(x) \le \lambda_b$ then
\begin{equation}\label{eq:bridge}
\varphi(x) \;\ge\; m_t(x) - \lambda_b \;>\; r_{\mathrm{safe}},
\end{equation}
that is, $x$ is not an unsafe cell, and equivalently $\{x \in U\} \subseteq \{x \in K_\alpha(t)\} \cup \{z_t(x) > \lambda_b\}$, and for a cell of the unseen stratum the keepout rule of Section~\ref{sec:problem}-B is itself the covered event and the same inclusion holds by the definition of miscoverage.
\textup{(ii)} Let Assumption~\ref{as:query} hold. Then $\Pr(X_t \in U) \le \Pr(X_t \in K_\alpha(t)) + \alpha$ at that sample, and for a fixed set of $N$ scored samples the expected unsafe fraction is at most the expected entry fraction plus $\alpha$.
\end{lemma}
\begin{proof}
(i) $x \notin K_\alpha(t)$ gives $m_t(x) > r_{\mathrm{safe}} + \lambda_b$ by~\eqref{eq:keepout} and $z_t(x) \le \lambda_b$ gives $\varphi(x) \ge m_t(x) - \lambda_b$, the step of Theorem~2 of Lindemann et al.~\cite{lindemann2023safe} for the constraint function $m_t(x) - r_{\mathrm{safe}}$ at Lipschitz constant $1$ and region radius $\lambda_b$, no independence between the two events being used.
(ii) For $F_t = \{X_t \in U,\ s_t(X_t) > \tau_b\}$, part (i) and the unseen rule give $\{X_t \in U\} \subseteq \{X_t \in K_\alpha(t)\} \cup F_t$, Assumption~\ref{as:query} gives $\Pr(F_t) \le \alpha$, and the union bound and linearity of expectation give the two claims.
\end{proof}
\begin{corollary}[unsafe flight events and keepout entries]\label{cor:flight}
For a flight scored at samples $t = 1, \dots, T$ let $C_T = \{\exists t: X_t \in U\}$, $E_T = \{\exists t: X_t \in K_\alpha(t)\}$ and $F_t = \{X_t \in U,\ s_t(X_t) > \tau_b\}$. Then $C_T \subseteq E_T \cup \bigcup_{t=1}^{T} F_t$, hence $\Pr(C_T) \le \Pr(E_T) + \Pr(\bigcup_{t} F_t)$ for every joint law of the flight.
\end{corollary}
\begin{proof}
Taking the union of the inclusions of Lemma~\ref{lem:contact}(i) at every $t$ and applying monotonicity and subadditivity of probability gives the result.
\end{proof}
A number for the second term needs per-sample coverage at $\alpha/T$ in the manner of Lindemann et al.~\cite{lindemann2023safe}, which the present calibration at $\alpha = 0.10$ does not provide, and in the replayed closed loop of Section~\ref{sec:eval} $C_T \subseteq E_T$ holds by construction.
Membership is evaluated at cell resolution, a centre lies within $\Delta\sqrt{3}/2$ of its cell centre and $\varphi$ is $1$-Lipschitz, therefore~\eqref{eq:bridge} implies body clearance at the exact position only when its middle term exceeds $r_{\mathrm{safe}}$ by $\Delta\sqrt{3}/2$, $0.087$~m here and not added to the margin, the same scope applying to part (ii) and Corollary~\ref{cor:flight}.

The probability in~\eqref{eq:coverage} is marginal, averaging the calibration cells and a fresh test cell of the calibration population, and because cells of one recording are correlated in space and time the held-out coverage of Section~\ref{sec:results} is computed blockwise, evidence about the certificate and not a proof of Assumption~\ref{as:exch}.
Assumption~\ref{as:query} asks the same event at the cell the vehicle occupies and does not follow from~\eqref{eq:coverage}, because the planner selects $X_t$ from the perceived map whose error the score measures, a gap Lindemann et al.\ close by assumption and Dixit et al.\ by a fixed sampled state set~\cite{dixit2024pwc}, neither established here.
A lower keepout-entry rate tightens the per-sample risk bound of Lemma~\ref{lem:contact}(ii), but does not itself establish a reduction in realised unsafe occupancy.

\section{Compliant Planning Inside the Keepout}\label{sec:method}
The planning layer, certified planning--monitor consistency (CPMC), receives the keepout $K_\alpha(t)$ in force at time $t$ and enforces one braking-distance inequality on it at plan time and, in a tightened form, at every command, on an unchanged planner~\cite{zhou2019fastplanner} (Fig.~\ref{fig:c3}). %
\textit{Problem.} Given $K_\alpha(t)$, a start and a goal, choose the reference and the command such that at every sample $t$ the executed centre satisfies $X_t \notin K_\alpha(t)$, the constraint, and the goal is reached within the mission time, the objective, while $K_\alpha(t)$ is never changed.

\subsection{The planner and its seam to the certificate}
The base planner searches a kinodynamic path on motion primitives, optimises a B-spline for smoothness and clearance and re-times it to the velocity and acceleration limits~\cite{zhou2019fastplanner}, and the keepout enters it at a single seam, the occupancy query of the search and the optimisation, which returns true for every cell of $K_\alpha(t)$.
Around this seam the method keeps a route search restricted to the free space $D(t) = W \setminus K_\alpha(t)$, a corridor around the route, a mission fence and a reference path along the middle of $D(t)$, without the stock planner's clearance term and reactive collision monitor. %

\subsection{One admissibility inequality}
Let $p$ and $v$ be the position and velocity of the tracked reference and $m_{\mathrm{tr}}$ the tracking margin, a sample quantile of this vehicle's measured distance between executed centre and reference, distinct from the map estimate $m(x)$.
For every cell $i$ of $K_\alpha(t)$ within the reach radius $R = m_{\mathrm{tr}} + v_{\max}^2 / (2 a_b) + a_b v_{\max} / (2 j_{\max}) + \Delta/2$, $j_{\max}$ being the command rate limit of Section~\ref{sec:method}-C, let $d_i(p, t)$ be the exact distance from $p$ to the cell's box, negative inside it, $g_i$ the unit vector from the box toward $p$ and $v_{n,i} = \max(0, -g_i^{\top} v)$ the closing speed. The admissibility margin is %
\begin{equation}\label{eq:admissibility}
h_i(p, v, t) = d_i(p, t) - m_{\mathrm{tr}} - \frac{v_{n,i}^2}{2 a_b},
\end{equation}
its minimum being $H(p, v, t) = \min_i h_i(p, v, t)$ and a reference state admissible when $H \ge 0$. %
\begin{figure}[t]\centering
\resizebox{\columnwidth}{!}{%
\begin{tikzpicture}[font=\large, >=Latex, lbl/.style={font=\normalsize}]
  \begin{scope}
    \fill[blue!25] (0,0) rectangle (0.8,0.8); \draw[blue!60!black] (0,0) rectangle (0.8,0.8);
    \node[lbl, blue!60!black, align=center] at (0.4,0.4) {cell $i$\\$\in K_\alpha(t)$};
    \draw[blue!50, dashed, rounded corners=2pt] (-0.3,-0.3) rectangle (1.1,1.1);
    \node[lbl, blue!60!black, anchor=north west] at (-0.3,-0.35) {shell of width $m_{\mathrm{tr}}$};
    \coordinate (p) at (3.0,2.1); \fill (p) circle (1.8pt); \node[lbl, anchor=south west] at ($(p)+(0.05,0.05)$) {reference state $(p, v)$};
    \coordinate (q) at (0.8,0.8);
    \draw[<->] (q) -- node[lbl, above, sloped] {$d_i$} (p);
    \draw[->, thick] (p) -- ++(-1.1,-0.33) node[lbl, below] {$v$};
    \draw[->, blue!60!black, thick] (q) ++(0.15,0.15) -- ++(0.45,0.26) node[lbl, right] {$g_i$};
    \draw[->, red!70!black, thick] (p) -- ++($-0.65*(0.866,0.5)$) node[lbl, left] {$v_{n,i}$};
    \draw[red!70!black, dashed] (p) circle (1.3); \node[lbl, red!70!black, anchor=west] at ($(p)+(0.9,0.95)$) {reach $R$};
    \node[lbl, align=left, anchor=north west] at (-0.3,-0.9) {$h_i = d_i - m_{\mathrm{tr}} - v_{n,i}^2/(2a_b) \ge 0$,\\braking at $a_b$ stops the reference before the shell};
  \end{scope}
  \begin{scope}[xshift=6.6cm, yshift=1.0cm]
    \draw[->] (0,0) -- (7.4,0) node[lbl, below] {time};
    \foreach \x in {0.3,1.7,3.1,4.5,5.9,7.3} {\draw[blue!60!black, thick] (\x,-0.12) -- (\x,0.12);}
    \node[lbl, blue!60!black, anchor=north east] at (6.6,-0.14) {certificate epochs $K_\alpha(t)$};
    \draw[orange!80!black, thick] (0.8,0) -- (0.8,-0.5);
    \node[lbl, orange!80!black, align=left, anchor=north west] at (0.0,-0.55) {plan accepted at $t_p$,\\$H \ge 0$ on every sample\\against $K_\alpha(t_p)$};
    \foreach \x in {0.9,1.0,...,7.3} {\draw[orange!60!black] (\x,-0.03) -- (\x,0.03);}
    \node[lbl, orange!80!black, align=left, anchor=north west] at (3.9,-0.55) {command filter, every $\Delta t$,\\$h^{F}_{i,k+1} \ge (1-\alpha_{\mathrm{cbf}}\Delta t)\,h^{F}_{i,k}$};
    \draw[decorate, decoration={brace, amplitude=3pt}] (1.7,0.5) -- (3.1,0.5) node[lbl, midway, above=2pt, align=center] {invariant between\\updates, Prop.~\ref{prop:filter}(i)};
    \draw[red!70!black, thick] (5.3,0) -- (5.3,0.5); \node[lbl, red!70!black, align=left, anchor=south west] at (5.35,0.5) {landing after an advance,\\recovery, counted};
  \end{scope}
\end{tikzpicture}}
\caption{The planning layer. Left, one keepout cell, $d_i$, $g_i$, $v_{n,i}$, reach $R$, shell of width $m_{\mathrm{tr}}$. Right, the two enforcement points on a timeline.}\label{fig:c3} %
\end{figure}
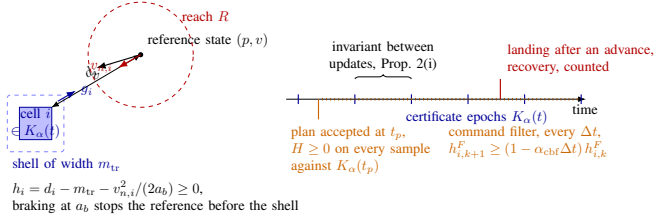
The constant $a_b$ is the planner's acceleration limit. A reference closing on cell $i$ at speed $v_{n,i}$ and decelerating at $a_b$ along the closing direction stops at distance $d_i - v_{n,i}^2/(2 a_b)$ from the cell. Thus $h_i \ge 0$ is necessary and sufficient for such a stop to end outside the shell of width $m_{\mathrm{tr}}$ around the keepout, and no second buffer is added at either enforcement point.
The form is the lane-keeping barrier of Ames et al.~\cite{ames2019cbf,ames2017cbfqp} composed on grid cells~\cite{harms2025composite,raja2024ogmcbf}, here protecting a time-varying conformal keepout. %

\subsection{Enforcement at plan time and at the command}
Each candidate of the planner is a reference path in $D$, a geometric pass moves the control points of any sample within $m_{\mathrm{tr}}$ of the keepout onto the reference route~\cite{chen2021fastrack} and refuses a candidate that cannot be moved clear, and the time profile is the fastest for which $H \ge 0$ holds at every sample against the keepout in force at planning time, built by a speed ceiling, a backward pass decelerating only toward cells ahead and a forward pass bounding the acceleration.
A candidate is accepted when $H \ge 0$ holds on all its samples at intervals of $\Delta t_{\mathrm{acc}}$ and re-timed rather than refused when it fails, a reference whose ceiling is zero somewhere has no route, a candidate that exhausts a cycle's re-timing budget is refused unless $H \ge 0$ already holds, and the planner plans from the state of the command filter below, not from the odometry. %
Because the keepout follows the map updates and the vehicle does not track its reference exactly, the second enforcement point filters the acceleration command at a fixed period $\Delta t$ between the planner and the autopilot.
The filter integrates a virtual reference state $(p_f, v_f)$ as a double integrator driven by the applied acceleration $a_f$, constant during each period, and forms a nominal acceleration
\begin{equation}\label{eq:nominal}
a_{\mathrm{nom}} = a_{\mathrm{cmd}} + k_p\, \mathrm{sat}(p_{\mathrm{cmd}} - p_f, e_p) + k_d\, \mathrm{sat}(v_{\mathrm{cmd}} - v_f, e_v),
\end{equation}
whose saturations bound the pull of the virtual state toward the plan.
At the command a cell's margin carries one more term,
\begin{equation}\label{eq:ramp}
h_i^{F} = h_i - \frac{a_b}{2 j_{\max}}\, v_{n,i},
\end{equation}
the distance covered while the command ramps from a non-braking value to full braking at the rate limit $j_{\max}$ per period, and $H^{F} = \min_i h_i^{F}$.
The filtered acceleration $a_f$ is the closest acceleration to $a_{\mathrm{nom}}$ that satisfies, for every cell within reach whose margin lies below an activation band $h_{\mathrm{band}}$, the one-step condition
\begin{equation}\label{eq:sampled}
h^{F}_{i,k+1} \;\ge\; (1 - \alpha_{\mathrm{cbf}}\, \Delta t)\, h^{F}_{i,k},
\end{equation}
where $h^{F}_{i,k+1}$ is predicted one period ahead for the applied acceleration and the current normal, inside the rate box $|a_f - a_{f,\mathrm{prev}}|_\infty \le j_{\max}\Delta t$ and the clip $|a_f| \le a_b$, a sequential closed-form projection on the most violated constraint in at most eight passes that modifies the applied reference and not the plan~\cite{garone2017governor}, and a period whose final command still violates a constraint is counted infeasible and replaced by maximal braking along $v_f$.
Three counted events interrupt the construction, a landing, some $h_i^{F} < 0$, starting a recovery push away from every violated cell until $H^{F} \ge 0$, a re-synchronisation, resetting the virtual state to a later admissible command, and a re-anchoring, resetting it to the odometry when the two drift apart beyond a bound.

\subsection{Compliance of the reference and the executed centre}
Geometric clearance alone does not ensure that a moving reference can stop before the keepout. CPMC enforces a consistent braking condition during trajectory acceptance and command filtering, addressing this mismatch as a mechanism for reducing keepout entries.
Where Lindemann et al.\ obtain feasibility of the constrained optimisation at every step, which places the modelled state outside the robustified constraint by construction, the requirement here splits into feasibility of the acceptance and of the filter, invariance of the filtered reference between certificate updates, tracking within the margin and compliance of the executed centre, the two statements below covering the planning links while tracking and compliance are measured, never assumed. %
\begin{proposition}[the command filter]\label{prop:filter}
\textup{(i) Invariance between updates.} Consider an interval between two certificate updates that contains no landing, re-synchronisation or re-anchoring, and assume \textup{(B1)} $H^{F}(p_f, v_f) \ge 0$ at its start, \textup{(B2)} no period of the interval is counted infeasible, \textup{(B3)} $0 < \alpha_{\mathrm{cbf}}\Delta t \le 1$, and \textup{(B4)} the one-period prediction in~\eqref{eq:sampled}, evaluated at the realised period, does not exceed the true value one period later. Then $h_i^{F} \ge 0$ for every cell of $K_\alpha$ at every period of the interval, and the reference never enters the keepout in it.
\textup{(ii) Compatibility away from the margin.} Let a sample of a delivered plan satisfy, for a cell $i$ ahead, $v_{n,i}^2 \le 2 a_{\mathrm{dec}}\,(d_i - m_{\mathrm{tr}})$ at a deceleration $a_{\mathrm{dec}} < a_b$, and write $\kappa = 1 - a_{\mathrm{dec}}/a_b$ and $r = a_b/(2 j_{\max})$. Then $h_i \ge \kappa\,(d_i - m_{\mathrm{tr}})$, and $h_i^{F} \ge 0$ whenever $d_i - m_{\mathrm{tr}} \ge 2 a_{\mathrm{dec}} r^2 / \kappa^2$. Moreover, if the nominal acceleration $a_{\mathrm{nom}}$ satisfies~\eqref{eq:sampled} for every active cell and lies inside the rate box and the clip, the projection returns it unchanged.
\end{proposition}
\begin{proof}
(i) For an active cell (B2) and (B4) give $h^{F}_{i,k+1} \ge (1 - \alpha_{\mathrm{cbf}}\Delta t)\, h^{F}_{i,k}$ and (B3) keeps a non-negative value non-negative~\cite{agrawal2017dcbf}, a cell outside the active set has $h_i^{F} > h_{\mathrm{band}}$ or lies beyond the reach and one period changes $h_i^{F}$ by at most $2 v_{\max}\Delta t + a_b \Delta t^2 + a_b^2 \Delta t / (2 j_{\max})$, below $h_{\mathrm{band}}$ at the nominal period of $0.01$\,s and the values of Section~\ref{sec:setup}, and induction along the periods gives the claim, entry into a cell needing $d_i \le 0$, hence $h_i^{F} < 0$. %
(ii) The bound gives $h_i \ge \kappa (d_i - m_{\mathrm{tr}})$ in~\eqref{eq:admissibility} and $h_i^{F} \ge \kappa s - r\sqrt{2 a_{\mathrm{dec}} s}$ at $s = d_i - m_{\mathrm{tr}}$ in~\eqref{eq:ramp}, non-negative exactly when $\sqrt{s} \ge r\sqrt{2 a_{\mathrm{dec}}}/\kappa$, and a nominal command that violates no constraint and lies inside the rate box and the clip is a fixed point of the projection.
\end{proof}
Assumption (B2) is not a theorem, the projection can be infeasible when the previous command is far from braking at a margin near zero, the count of infeasible periods is its measured counterpart, and the proposition is sampled~\cite{breeden2022sampled}, compliance being reported at the checker's rate.
Part (ii) applies wherever the re-timing pass limits the delivered profile at $a_{\mathrm{dec}} < a_b$, its threshold being $0.02$~m at the values of Section~\ref{sec:setup}, whereas the decay condition of~\eqref{eq:sampled} is stronger than $h^{F} \ge 0$ and not implied by part (ii), its continuous form holding along the same profile only beyond about $0.6$~m from the margin, therefore the filter is expected to change the command late in an approach to a cell ahead, the share of changed periods being the measured counterpart. %
The vehicle follows the reference at an along-track lag $\tau$ and the tracking residual is $e_\tau(t) = |x(t) - p_f(t - \tau)|$, a distance function is $1$-Lipschitz, therefore $\mathrm{dist}(x(t), K_t) \ge \mathrm{dist}(p_f(t - \tau), K_t) - e_\tau(t)$, and when the keepout did not advance in $[t - \tau, t]$ and $H^{F} \ge 0$ held at $t - \tau$ the executed centre is outside the keepout whenever $e_\tau(t) < m_{\mathrm{tr}}$, an advance inside the lag window being counted as an entry when it covers the lagging vehicle.
The margin is a quantile, not a bound, and the rate of samples at which $e_\tau(t) \ge m_{\mathrm{tr}}$ is reported beside the lag. %
\subsection{Admissible motion}
\begin{lemma}[directional admissibility along a corridor]\label{lem:gaps}
In a straight corridor the wall cells abeam have normals orthogonal to the velocity, $v_{n,i} = 0$ for every such cell and the condition $H \ge 0$ on them reduces to $d_i \ge m_{\mathrm{tr}}$ at every speed up to $v_{\max}$.
A plan-time law that admits a sample only when the nearest-cell clearance $\rho$ satisfies $\rho \ge c(v) + v\,t_r + v^2/(2a)$, for a reaction time $t_r \ge 0$, a deceleration $a \le a_b$ and a tube $c(v) \ge m_{\mathrm{tr}}$, admits in any geometry a subset of the samples at which $H \ge 0$, and in the corridor every sample at which $m_{\mathrm{tr}} \le d_i < c(v) + v\,t_r + v^2/(2a)$ is admitted by the condition on the abeam cells alone.
\end{lemma}
\begin{proof}
Substitution of $v_{n,i} = 0$ in~\eqref{eq:admissibility} gives the first claim. For every cell $d_i \ge \rho$ and $v_{n,i} \le v$, hence $h_i \ge \rho - m_{\mathrm{tr}} - v^2/(2a_b) \ge 0$ whenever $\rho \ge c(v) + v t_r + v^2/(2a)$, $t_r \ge 0$, $a \le a_b$ and $c(v) \ge m_{\mathrm{tr}}$, the inclusion, and for wall-parallel motion the nearest-cell clearance is the lateral distance $d_i$, where the remaining band is admitted by the first claim only.
\end{proof}
The admissible-set inclusion concerns the stated isotropic law, and the experimental law-only arm departs from its hypotheses by the braking constant and creep rule of Section~\ref{sec:setup}-B. The inclusion holds generally, whereas the corridor-band example applies only to abeam cells. By feasible-set monotonicity, the directional condition cannot worsen the optimal objective of an ideal planner on the same keepout relative to the stated law. The deployment pattern's soft clearance and reactive stop spend clearance the certificate has already budgeted, a mechanism of the measured reach difference. %
Invariance of the reference and admissibility of more motions do not establish arrival, because a vehicle that stops remains compliant, the motion of the certificate between updates is only counted, and none of these statements orders two planning stacks, the comparison on the same cells being the registered experiment of Section~\ref{sec:results}.

\section{Evaluation}\label{sec:eval}\label{sec:setup}\label{sec:results}

\subsection{Certificate validation on recorded data}
The certificate is validated on an indoor smoke dataset recorded by a tracked ground robot carrying a 16-ring LiDAR and a marker-based ground-truth system~\cite{upo_nix}.
Bags 1 to 5 are split in severity-stratified blocks of 10~s per bag, every test block apart in time from the calibration blocks, and the class edges are fixed on the calibration blocks alone.
Coverage of the certified event is reported per severity class against the bound~\eqref{eq:coverage}, its intervals being 60\,s cluster bootstrap intervals at 95\,\%. %
No claim of this track crosses to the closed loop.

On the held-out recordings the per-class certificate covered the certified event on 0.900 [0.885, 0.913] of unsafe cells pooled (Fig.~\ref{fig:c1}(c)), the interval of every class containing the target, where one pooled margin covered the foggiest quartile below the target and the clearest above it. %
Table~\ref{tab:cert} reports the leave-one-bag-out folds, keepout shares, pooled-margin and uncalibrated 0.3~m buffer baselines, and simulator-calibrated certificate on held-out episodes. %
\begin{table}[t]
\caption{Certificate rows not in Fig.~\ref{fig:c1}(c). Set membership, the unsafe cell inside the keepout, scored on the same held-out blocks.}
\label{tab:cert} %
\centering\scriptsize\setlength{\tabcolsep}{2.3pt}%
\begin{tabular}{@{}lll@{}}
\toprule
Row & Coverage [95\,\%] & Keepout share \\ %
\midrule
per-class, certified event, pooled & 0.900 [0.885, 0.913] & 0.766 [0.746, 0.785] \\
pooled margin, certified event, foggiest & 0.873 [0.826, 0.915] & 0.778 \\
pooled margin, certified event, clearest & 0.913 [0.894, 0.933] & 0.778 \\
leave-one-bag-out, foggiest & 0.888 [0.850, 0.925] & -- \\
leave-one-bag-out, other quartiles & 0.900 to 0.902 & -- \\
per-class, set membership & 0.918 [0.904, 0.930] & 0.766 [0.746, 0.785] \\
buffer 0.3\,m (arm S~0.3), set membership & 0.565 [0.531, 0.596] & 0.264 [0.245, 0.282] \\
sim.-calibrated, held-out tunnel & 0.964 [0.958, 0.970] & 0.734 \\
sim.-calibrated, held-out caves & 0.877 [0.865, 0.889] & 0.742 \\
\bottomrule
\end{tabular}
\end{table}

\subsection{Closed-loop evaluation}
\textit{Apparatus and certificate.} The planning layer is evaluated in closed loop on a quadrotor model in a simulated cave built on FlightBench~\cite{flightbench}.
The simulated LiDAR reproduces the recorded sensor's ring geometry, and a fog model fitted to recorded smoke sweeps produces degradation~\cite{rasshofer2011weather,hahner2021fog}, no overall fidelity being claimed. %
The certificate in force is a replayed process, the union of a static keepout, namely a ladder shell of fixed certified clearance built from ground-truth route geometry, and an advance sequence fitted to certificate fields generated by the frozen perception pipeline on calibration flights. The fog and LiDAR models therefore affect the results only by way of those flights.
The advances reflect sensing conditions and mapping representation alike.
The sequence at each pose and replicate is identical for every arm and independent of their actions, and no live mapping or margin selection runs during flights.
The certificate is fitted from two calibration sources, simulator-calibrated margins from simulated cave LiDAR against simulator ground truth and real-data-transferred margins from the recordings of Section~\ref{sec:eval}-A at the same class edges.
The calibration source is the only variable between regimes, and neither claims coverage in the physical world.
By construction, the unsafe set of the route, the cells within the body radius of a true wall, lies inside the static keepout of either setting by at least 0.465\,m, and the advance process only adds cells, therefore the keepout in force contains that set at every epoch of every flight and every contact sample is an entry sample.

\begin{figure}[t]\centering
\includegraphics[width=0.75\columnwidth]{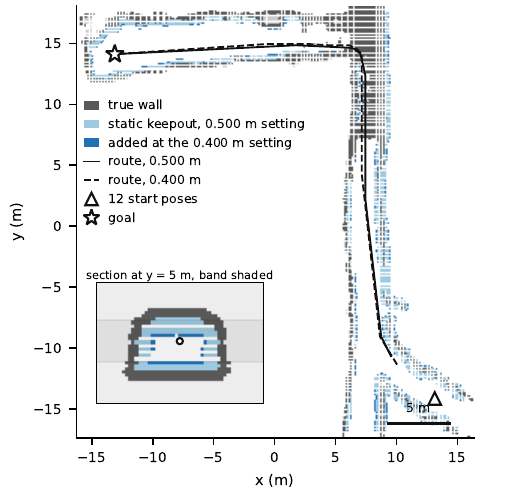} %
\caption{The route. Top, plan view, keepouts of both settings, reference routes, start poses, goal. Bottom, cross-section at $y = 5$\,m, flight band shaded, both shells.}\label{fig:cave} %
\end{figure}
\textit{Design and arms.} The route is a tunnel map (Fig.~\ref{fig:cave}) at two certified-clearance settings, 0.500 and 0.400~m, and two degradation levels, clear air and fog of extinction 0.72.
Each combination has 12 start poses within 0.15\,m of one another, flown in two replicates, giving 96 cells, each flown once by each of seven arms, 1344 flights in the two regimes.
Every arm shares the poses, seeds, goal and acceleration limit of 2.0~m/s$^2$, and all but $A'$ 1.10 the speed limit of 1.5~m/s. %
The constants are $\Delta t = 0.01$~s, $\alpha_{\mathrm{cbf}} = 2$~s$^{-1}$, $h_{\mathrm{band}} = 0.05$~m, $j_{\max} = 20$~m/s$^3$, $a_{\mathrm{dec}} = 1.0$~m/s$^2$, $k_p = k_d = 4$, $e_p = 0.25$~m, $e_v = 0.25$~m/s and $\Delta t_{\mathrm{acc}} = 0.01$~s. %
M, the proposed stack, is KEPT's planning layer. %
The law-only arm applies an isotropic plan-time law of the form of Lemma~\ref{lem:gaps} at a braking constant of 5~m/s$^2$, outside the lemma's hypothesis $a \le a_b$, plus a low-clearance creep rule permitting speeds of 0.10 to 0.30~m/s, and omits the command filter. The filter-only arm is the command filter alone in front of the $A'$ planner, and the two isolate what each adds. %
$A'$ 0.15 is the stock planner on the certified map at a soft clearance of 0.15~m and its own emergency stop, the conformal deployment pattern and the claim's comparator.
$A'$ 0.25 is a second point on its clearance curve, and $A'$ 1.10 caps its speed at 1.10~m/s, testing whether speed explains the result.
S 0.3 is the stock planner at a fixed inflation of 0.3~m and no certificate, assessing the certificate's contribution.
The comparator's clearance and monitor buffer were tuned on maps disjoint from the route and frozen before any flight.
A start filter fixed before any flight removes, for every arm alike, cells whose resting footprint intersects the certificate at the start.
It removed three fog cells at the 0.500 setting in the simulator-calibrated regime and none in the transferred regime, leaving 93 and 96 cells for paired analysis.
Each flight ends 5~s after reaching the goal or at 150~s. %

\textit{Outcomes.} Compliance is the executed centre staying out of the certificate in force at every sample, entries being attributed to the planner or to the certificate's motion, safety the executed centre staying out of the true unsafe set $U$, and reach the ground-truth position coming within 1.0~m of the goal within 150~s. %
The primary compliance count is the number of flights containing an entry, reported beside the entry episodes, their depth and duration.

\textit{Analysis.} The claim is decided per setting and level and pooled.
The pre-registered criterion is that M wins against $A'$ 0.15 when its entry flights and entry episodes are not more than those of $A'$ and its reached flights are strictly more, and a strict win, fewer on both entry counts, is reported beside it. %
Differences are paired by pose, intervals are 95\,\% pose-bootstrap intervals, and $p$ values two-sided sign-flip tests.
Pre-registration, run matrix, real-time validity rule, incident policy and host assignment were fixed and hashed before any evaluation flight. %

\subsection{Compliance and reach against the deployment pattern}
\begin{table*}[t]
\caption{Stack M against the deployment pattern $A'$ 0.15 on the matched cells of each regime, M~/~$A'$. EF entry flights, EP entry episodes, R goals reached, $\Delta$EF, $\Delta$R paired differences M minus $A'$, 95\,\% interval and $p$, Won the criterion of Section~\ref{sec:eval}-B.}
\label{tab:claim} %
\centering\footnotesize\setlength{\tabcolsep}{4pt}%
\begin{tabular}{llrcccllc}
\toprule
Regime & Stratum & $n$ & EF & EP & R & $\Delta$EF [95\,\%] $p$ & $\Delta$R [95\,\%] $p$ & Won \\
\midrule
Simulator-calibrated & clear 0.500 & 24 & 1~/~1 & 2~/~6 & 13~/~21 & $+0.00$ [-0.25, +0.25] 1.00 & $-0.67$ [-1.33, +0.08] 0.16 & no \\
 & clear 0.400 & 24 & 2~/~1 & 6~/~1 & 19~/~13 & $+0.08$ [-0.17, +0.33] 1.00 & $+0.50$ [-0.08, +1.08] 0.22 & no \\
 & fog 0.500 & 21 & 0~/~3 & 0~/~5 & 14~/~14 & $-0.25$ [-0.50, +0.00] 0.25 & $+0.00$ [-0.42, +0.42] 1.00 & no \\
 & fog 0.400 & 24 & 0~/~4 & 0~/~4 & 19~/~13 & $-0.33$ [-0.58, -0.08] 0.12 & $+0.50$ [+0.08, +0.92] 0.11 & yes \\
 & pooled & 93 & 3~/~9 & 8~/~16 & 65~/~61 & $-0.13$ [-0.25, +0.00] 0.15 & $+0.08$ [-0.19, +0.38] 0.70 & yes \\
\midrule
Transferred & clear 0.500 & 24 & 2~/~3 & 2~/~7 & 15~/~20 & $-0.08$ [-0.25, +0.00] 1.00 & $-0.42$ [-0.92, +0.08] 0.25 & no \\
 & clear 0.400 & 24 & 0~/~1 & 0~/~1 & 19~/~19 & $-0.08$ [-0.25, +0.00] 1.00 & $+0.00$ [-0.50, +0.42] 1.00 & no \\
 & fog 0.500 & 24 & 2~/~6 & 2~/~10 & 17~/~16 & $-0.33$ [-0.67, +0.00] 0.22 & $+0.08$ [-0.58, +0.67] 1.00 & yes \\
 & fog 0.400 & 24 & 1~/~8 & 1~/~8 & 16~/~10 & $-0.58$ [-1.00, -0.17] 0.06 & $+0.50$ [+0.00, +1.08] 0.19 & yes \\
 & pooled & 96 & 5~/~18 & 5~/~26 & 67~/~65 & $-0.27$ [-0.42, -0.13] 0.004 & $+0.04$ [-0.23, +0.31] 0.89 & yes \\
\bottomrule
\end{tabular}
\end{table*}

\begin{table}[t]
\caption{Pooled counts, M against every other arm, M~/~arm, columns as in Table~\ref{tab:claim}.}
\label{tab:comparators} %
\centering\footnotesize%
\resizebox{\columnwidth}{!}{%
\begin{tabular}{llcccc}
\toprule
Regime & Arm & EF & EP & R & Won \\
\midrule
Simulator-calibrated & deployment pattern $A'$ 0.15 & 3~/~9 & 8~/~16 & 65~/~61 & yes \\
 & $A'$ 0.25 & 3~/~10 & 8~/~20 & 65~/~66 & no \\
 & $A'$ 1.10 & 3~/~8 & 8~/~11 & 65~/~70 & no \\
 & law only & 3~/~2 & 8~/~2 & 65~/~63 & no \\
 & filter only & 3~/~5 & 8~/~7 & 65~/~65 & no \\
 & S 0.3 & 3~/~90 & 8~/~223 & 65~/~35 & yes \\
\midrule
Transferred & deployment pattern $A'$ 0.15 & 5~/~18 & 5~/~26 & 67~/~65 & yes \\
 & $A'$ 0.25 & 5~/~17 & 5~/~18 & 67~/~69 & no \\
 & $A'$ 1.10 & 5~/~9 & 5~/~10 & 67~/~65 & yes \\
 & law only & 5~/~5 & 5~/~7 & 67~/~63 & yes \\
 & filter only & 5~/~5 & 5~/~5 & 67~/~63 & yes \\
 & S 0.3 & 5~/~93 & 5~/~276 & 67~/~64 & yes \\
\bottomrule
\end{tabular}}
\end{table}

 \begin{figure}[t]\centering
\includegraphics[width=0.55\columnwidth]{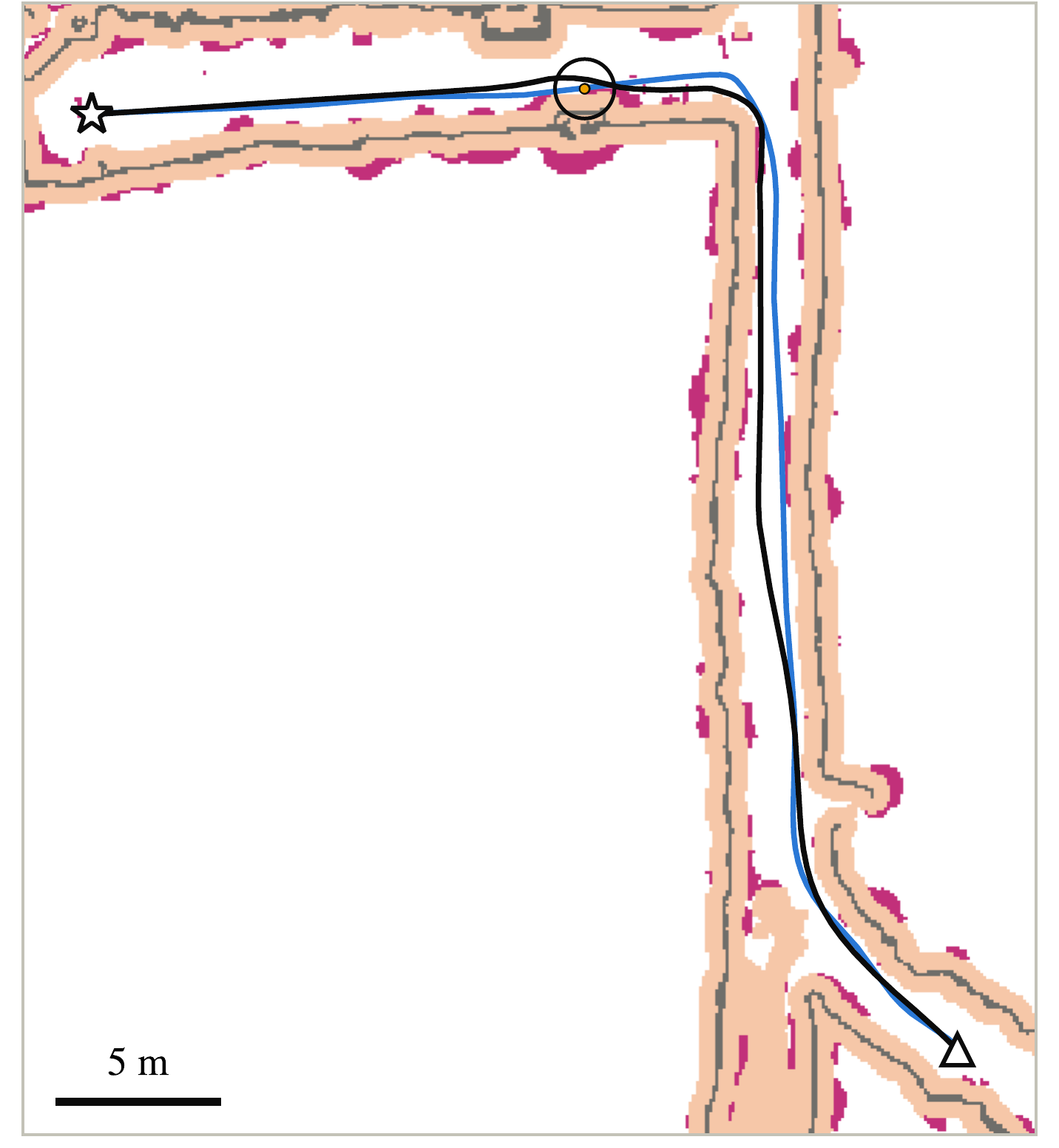} %
\caption{One selected cell, fog at the 0.400 setting, simulator-calibrated regime, where A$'$ 0.15 entered the keepout and M reached the goal without entry, not a typical flight. Grey, true walls. Light, static keepout. Saturated, advance cells in force in either flight. Black, M. Blue, A$'$ 0.15. Ringed dots, entries. Triangle, start. Star, goal.} %
\label{fig:f1map} %
\end{figure}
Table~\ref{tab:claim} reports entry flights, entry episodes and goals reached for the proposed stack and the deployment pattern on matched cells in each regime.
Fig.~\ref{fig:f1map} shows one example.
On the 189 pooled cells the stack entered in 8 flights against 27 and reached 132 goals against 126, 69.8 versus 66.7 percent, an observed gain of 3.2 percentage points, positive in both regimes, the per-regime intervals including zero.
The pre-registered criterion holds pooled in both regimes and in both fog 0.400 strata but fails in clear air at 0.500, where the stack reached 15 against 20 and 13 against 21 goals as explained in Section~\ref{sec:eval}-D. %
The six reversals are entries one or two map cells deep, 0.05 to 0.10\,m, lasting at most 0.93\,s, in cells where the deployment pattern made no entry.
Every stack entry stayed within this depth and duration, whereas 7 of the deployment pattern's 27 entry flights exceeded 0.10\,m and one left the mapped volume after wall contact.
The checker's fail-closed convention counts that flight's remaining 59\,s as inside the keepout, a scoring rule, not observed obstacle-cell penetration.
The three deployment-pattern arms account for all six reversals, four in the simulator-calibrated regime.

\subsection{Enforcement points, cost and failure cases}
The law-only and filter-only arms each recover the compliance gain, entering in 2 and 5 flights in the simulator-calibrated regime and 5 and 5 in the transferred regime, versus 3 and 5 for the full stack. %
The pair adds reach against both ablations in the transferred regime (67 against 63 and 63) but not against the filter-only arm in the simulator-calibrated regime (65 against 63 and 65).
The full stack therefore does not dominate both ablations on both measures in both regimes, and its claim rests on the comparison against the deployment pattern (Table~\ref{tab:comparators}).

The filter changed the command in 6.3 and 5.7 percent of its periods and found 5.4 and 4.6 percent infeasible, where maximal braking was applied and the binding barrier stayed non-negative in every case.
Certificate advances put the reference inside the margin in 151 and 211 landing episodes, 135 and 180 of which recovered by the push along the violated normals and 16 and 31 by re-synchronisation, and no flight ended in an unrecovered landing.
The aligned tracking residual, the distance to the reference 0.46\,s behind as the margin was fitted, exceeded the margin of 0.04\,m on 9.6 and 10.6 percent of samples, twice the share the margin's quantile implies, at a 95th percentile of 0.052 and 0.055\,m against 0.032\,m on the development flights. The same-time distance to the published reference, including the along-track lag, had a median of 0.40 and 0.39\,m while moving. %

Of 201 valid flights censored at 150\,s, 189 were stationary, median speed below 0.05\,m/s in the final 20\,s, and every censored comparator flight stalled. %
The remaining 12 were stack flights still moving toward the goal (4 transferred, 8 simulator-calibrated) at 0.14 to 0.68\,m/s, 1.5 to 22.5\,m away, each on an accepted plan less than 3.5\,s old, and seven were within 8.5\,m.
Their eventual arrival was not evaluated.
Another 15 censored stack flights, all in clear air, never obtained an accepted plan (the start-stall case described below).
These two planning-layer behaviours, rather than properties of the certificate itself, account for the stack's reach losses and make censoring asymmetric between arms.
At start poses whose route passes inside a keepout cell's tracking margin, the acceptance test refuses every candidate at rest, and in clear air the stall persists.
A slower vehicle meets more certificate advances, and an advance covered the stack in flight in cells the deployment pattern had already passed.

\section{Conclusion and Limitations}\label{sec:conclusion}
This paper selected conformal keepout margins per viewing direction from a label-free severity estimate.
Staying outside the keepout preserves clearance when the calibrated error bound holds.
Planning and control never altered these margins and let the keepout follow the perceived map.
Both enforced the same braking-distance condition.
Its directional form was shown to be no more restrictive than any isotropic braking law that brakes no harder and keeps at least the same margin.
In 189 paired simulated cases the stack entered the keepout in 8 flights against 27 for the usual deployment pattern and reached 132 goals against 126.
Fewer entries and more goals were observed in both calibration regimes.

The coverage guarantee holds per map cell at the moment a keepout is in force.
A bound for a whole flight does not follow from it alone.
To isolate the planning layer the closed loop replayed keepouts fitted to the proposed pipeline's certificates from calibration flights.
Margins calibrated on smoke recordings carry no coverage guarantee in simulation.
The simulator-calibrated regime is reported beside them for that reason.
The evidence is one route, one fog level beside clear air, one planner and one quadrotor model in simulation.
All comparators use the same underlying planner and map.
Filter-infeasibility events were recorded, but a feasible command was not guaranteed.
Future work should address two failure modes.
The vehicle can stall at the start when every candidate plan passes inside the keepout's tracking margin.
In flight a keepout update can cover a slowly moving vehicle.
\bibliographystyle{IEEEtran}
\bibliography{IEEEabrv,refs}

\end{document}